\documentclass[sn-mathphys]{sn-jnl}
\usepackage{amsmath}
\DeclareMathOperator*{\argmax}{arg\,max}

\jyear{2023}%

\theoremstyle{thmstyleone}%
\newtheorem{theorem}{Theorem}
\theoremstyle{thmstyletwo}%

\theoremstyle{thmstylethree}%
\usepackage{tablefootnote}
\usepackage{commath}
\usepackage{verbatim}
\usepackage{amsmath,amssymb,bm}
\usepackage{ulem}
 \newcommand{\sgn}{\operatorname{sgn}}

\usepackage{xspace}
\newcommand*{\eg}{\textit{e.g.},\@\xspace}
\newcommand*{\ie}{\textit{i.e.},\@\xspace}
\def\etal{\textit{et al.}\@\xspace}

\newcommand{\tabref}[1]{Tab.~\ref{#1}}
\newcommand{\algoref}[1]{Alg.~\ref{#1}}
\newcommand{\equref}[1]{Eq.~\ref{#1}}
\newcommand{\theoremref}[1]{Theorem.~\ref{#1}}
\newcommand{\sectionref}[1]{Sec.~\ref{#1}}

\newtheorem*{theorem*}{Theorem}

\usepackage{dsfont}

\begin{document}

\title[A3T]{A3T: Accuracy Aware Adversarial Training}


\author*[1]{\fnm{Enes} \sur{Altinisik}}\email{ealtinisik@hbku.edu.qa}
\author[1]{\fnm{Safa} \sur{Messaoud}}\email{smessaoud@hbku.edu.qa}
\author[1]{\fnm{Husrev Taha} \sur{Sencar}}\email{hsencar@hbku.edu.qa}
\author[1]{\fnm{Sanjay} \sur{Chawla}}\email{schawla@hbku.edu.qa}

\affil[1]{\orgdiv{Qatar Computing Research Institute}, \orgname{HBKU}, \orgaddress{ \state{Doha}, \country{Qatar}}}

\abstract{Adversarial training has been empirically shown to be more prone to overfitting than standard training. The exact underlying reasons still need to be fully understood.  In this paper, we identify one cause of overfitting related to current practices of generating adversarial samples from misclassified samples.  To address this, we propose an alternative approach that leverages the misclassified samples to mitigate the overfitting problem. We show that our approach achieves better generalization while having comparable robustness to state-of-the-art adversarial training methods on a wide range of computer vision, natural language processing, and tabular tasks.}

\keywords{Adversarial Training; Overfitting in Adversarial Training; Misclassification Aware Adversarial Training}



\maketitle
\section{Introduction}
\label{sec:introduction}

Deep learning achieved excellent generalization on a wide variety of tasks \cite{jumper2021highly, fawzi2022discovering, ramesh2021zero, he2015delving, vaswani2017attention}. However, deepnets are also known to be vulnerable to adversarial attacks \cite{papernot2017practical, kurakin2016adversarial, kurakin2018adversarial, schonherr2018adversarial}, due to the non-smooth nature of their associated loss landscapes. Adversarial training (AT) \cite{goodfellow2014explaining} is established as the most efficient mechanism for defending against these attacks and consists of training on adversarially perturbed samples instead of on the original ones. Yet, AT is also more prone to overfitting \cite{madry2017towards,schmidt2018adversarially,rice2020overfitting, dong2021exploring, yu2022understanding,stutz2022understanding} and therefore results in a lower generalization than standard training. The reasons behind this are still under-explored. Nevertheless, several classical and modern deep learning remedies for overfitting, \eg regularization \cite{zhang2022alleviating, miyato2015distributional, gan2020large, chen2020robust} and data augmentation \cite{rice2020overfitting,wu2020adversarial,gowal2020uncovering} have been proposed without a full understanding of the root causes.

Recently, Dong \etal \cite{dong2021exploring} show that overfitting is due to augmenting the training data with `hard' to classify adversarial samples that incentivize a complicated adjustment of the decision boundaries. The authors argue that these `hard' adversarial samples are generated from samples close to the decision boundary with noisy or inappropriate ground-truth one-hot encoded labels. They show that previously proposed techniques for label and weight smoothing \cite{pang2020bag, huang2020self, chen2020robust} can alleviate this issue.

\begin{figure}
    \centering
    \includegraphics[width=0.9\textwidth, trim = 0cm 4.5cm 1.5cm 3.5cm, clip]{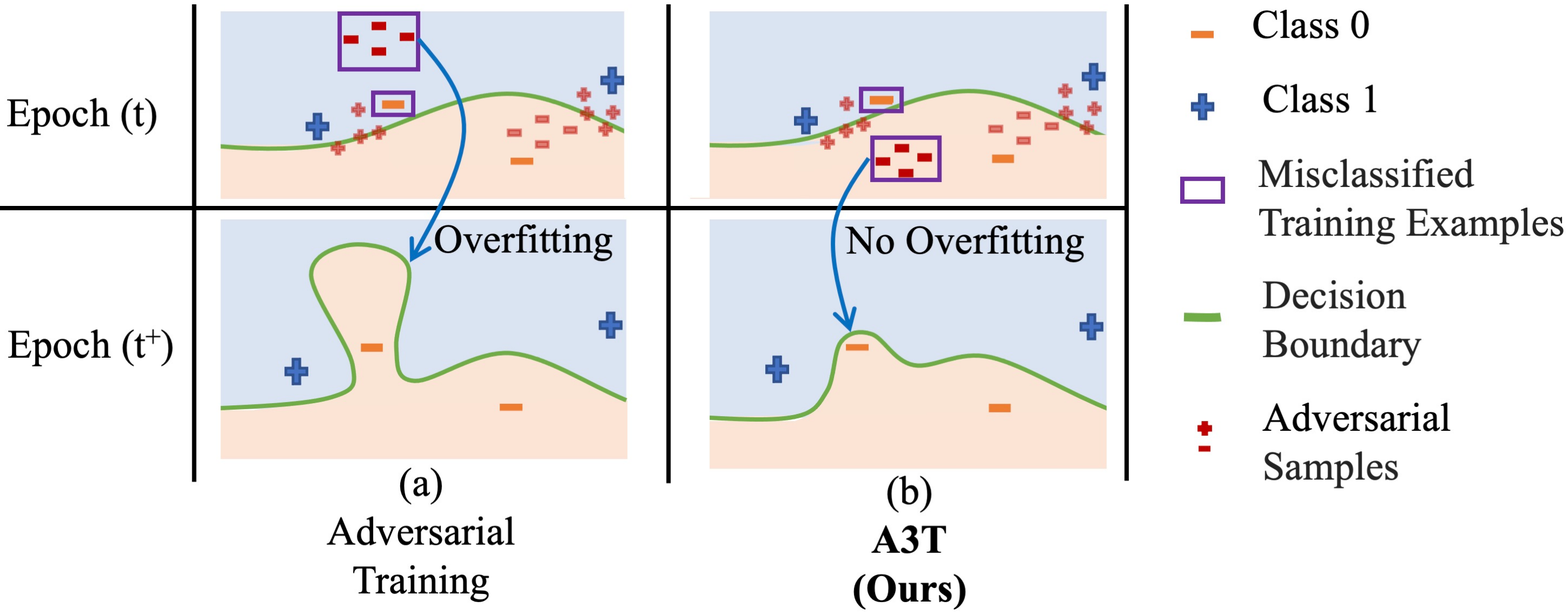}
    \caption{(a) In adversarial training, adversarial examples generated from misclassified samples are always placed maximally away from the decision boundary, making the model more prone to overfitting; (b) A3T adjusts the adversarial objective to generate adversarial examples close to the decision boundary and thus encourages learning smoother loss landscapes which achieves better generalization.}
    \label{fig:a3t}
\end{figure}

In this paper, we identify another category of `hard' adversarial samples. Specifically, we show that adversarial samples generated from \textit{misclassified samples} in every iteration of the training are, counter-intuitively, always placed maximally away from the decision boundary, as illustrated in \figref{fig:a3t}. Hence, they induce a large loss leading to substantial perturbations in the local region of the misclassified point. This problem is more exacerbated in high-capacity models, as they have the flexibility of making local changes to a decision boundary. To address this problem, we propose Accuracy Aware Adversarial Training (A3T) which controls the generation of adversarial samples in a manner that is aware of the current predicted label of the original training example. While Wang \etal \cite{wang2019improving} and Ding \etal \cite{ding2018mma} proposed formulations that generate adversarial samples differently for misclassified samples, their insights were imperially driven, and they do not identify nor directly address the root cause for poor generalization, \ie adversarial samples from misclassified samples are generated maximally away from the decision boundary. Differently, A3T proposes a simpler fix that directly addresses the root cause, \ie places adversarial samples from misclassified samples closer to the decision boundary.

We show that our A3T improves upon previously adopted techniques for mitigating overfitting in AT and achieves better generalization while having comparable robustness to state-of-the-art models on both toy experiments as well as on computer vision, natural language processing, and tabular applications. This hence demonstrates that we have identified a root cause for overfitting that is unaddressed in the literature.

\section{Related work}
\label{sec:related_work}

\subsection{Adversarial Training}
Adversarial Training (AT) \cite{goodfellow2014explaining} is a defense mechanism against adversarial attacks which aim at fooling a trained supervised machine learning model by adding imperceptible perturbations to the inputs to cause misclassification. Modern day deepnets are particularly vulnerable to such attacks. Due to the non-smooth nature of their loss-landscapes, it is possible to find, at almost any input sample, the direction of a steep gradient (adversarial direction) where perturbations along such directions lead to high loss and potentially a different prediction \cite{szegedy2013intriguing}. To address this, in AT, the model is trained to correctly classify perturbed versions of the inputs. Formally, AT is formulated as a min-max program, searching for the best parameters $\theta$ of a classifier $f_{\theta}$ under the worst-case perturbation $\delta$ applied to an input $x$ \cite{madry2017towards}, \ie
\begin{equation}
\min_{\theta} \mathbf{E}_{(x,y)\sim\mathcal{D}}\big[ \max_{\delta \in \Delta} \ell(f_{\theta}(x+\delta), y )  \big],
\label{eq:adversarial_loss}
\end{equation}
where $\ell$ is a loss function (\eg cross-entropy loss) and $\Delta$ is a constrained set of imperceptible perturbations. The outer-optimization program is classically solved using gradient descent. Popular approaches for solving the constrained inner maximization program include Fast Gradient Sign Methods (FGSM) \cite{szegedy2013intriguing,goodfellow2014explaining} and Projected Gradient Descent (PGD) \cite{madry2017towards}. In FGSM, the perturbation $\delta$ is computed via:
\begin{equation}
    \delta = \alpha \cdot sgn(\nabla_{x} \ell( f_{\theta}(x), y  ) ),
\end{equation}
with $\alpha$ being the learning rate and $sgn$ the sign operator. PGD is a multi-step variant of FGSM. It starts at randomly initialized perturbations in the feasible set $\Delta$ and iteratively applies a gradient ascent update followed by a projecting onto $\Delta$:
\begin{equation}
    \delta^{k}:=\Pi_{\Delta} \left( \delta^{k-1} + \alpha \nabla_{x}\ell(f_\theta(x+\delta^{k-1}),y) \right)
    \label{eq:PGD},
\end{equation}
where $\Pi_{\Delta}$ is the projection operator into $\Delta$ and $k$ is the number of steps the algorithm runs. 

AT is only defined for correctly classified samples. However, adversarial examples are generated following the same procedure for both correctly classified and misclassified samples. We will show that this makes the model more prone to overfitting.

\subsection{Overfitting in AT}
Different extension to standard AT objective (\equref{eq:adversarial_loss}) have been proposed to alleviate overfitting, including Local Distributional Smoothness (LDS) \cite{miyato2018virtual, zhang2019theoretically, villa, miyato2018virtual}, Margin Maximization (MM) \cite{balaji2019instance,cheng2020cat}, Label-Smoothing Regularization (LSR) \cite{cheng2020cat} and Misclassification Awareness (MA) \cite{wang2019improving}. Objectives of notable methods from these different families are presented in \tabref{tab:cmp}. 
However, all these attempts try to fix overfitting without identifying the root causes behind it. Naturally, creating a truly robust model would require a deeper understanding of the source of this vulnerability. Recently, Dong \etal \cite{dong2021exploring} have shown that overfitting is caused by generating `hard' to classify adversarial samples that are memorized by the network, \ie the network adjusts its boundaries in a complex way around these samples in order to produce correct classifications. The authors argue that these `hard' adversarial samples are generated from samples close to the boundaries and hence their associated one-hot encoding ground truth label is likely to be inaccurate or noisy. To fix this, the authors propose label smoothing of the adversarial samples. Yu \etal \cite{yu2022understanding} empirically show that small-loss data is responsible for overfitting and propose a minimum loss constrained adversarial training procedure that increases the loss of small-loss data. 

\noindent\textbf{Local Distributional Smoothness (LDS)} extends the robust optimization problem in \equref{eq:adversarial_loss} with a regularization term that encourages the label distribution around each sample point to be locally smooth. This is achieved via minimizing the KL divergence between the label distribution of the original sample and the one of the corresponding adversarially generated example ~\cite{miyato2018virtual, zhang2019theoretically}. Intuitively, this results in forcing the decision boundary to be sufficiently away from the training samples ~\cite{miyato2018virtual, zhang2019theoretically}. AT objectives from this family include TRADES \cite{zhang2019theoretically} and VILLA \cite{villa} (Lines 2 and 3 in Table \ref{tab:all}). A major drawback of this approach is that it equally reinforces correct and incorrect predictions, \ie it encourages the adversarial examples from misclassified samples to have the same label as the original (misclassified) samples. This results in a drop in generalization. \\
\noindent\textbf{Margin Maximization (MM)} methods propose to individually tune the perturbation strength for each sample. Intuitively, adopting a fixed level of perturbation for all the samples can potentially result in pushing adversarial examples further away from the boundary, causing them to be mixed with samples of other classes. This would force the model to learn a non-smooth and complex decision boundary. AT bjectives from this class include MMA \cite{ding2018mma}, IAAT \cite{balaji2019instance}, and CAT \cite{cheng2020cat} (Lines 4-6 in \tabref{tab:all}). Note that MM methods only generate adversarial examples from correctly classified samples.\\
\noindent\textbf{Label-Smoothing Regularization (LSR)}: One cause for overfitting is the use of one-hot encoded ground-truth labels for adversarial samples \cite{dong2021exploring}, which essentially forces the model to assign an overconfident probability of one to all samples of a given class. This can be noisy or inappropriate in case of adversarial examples as they are deliberately generated by pushing close-to-the-boundary samples over to boundary to the side of another class. To address this, Cheng \etal \cite{cheng2020cat} propose applying a label-smoothing regularization depending on the perturbation tolerance of each sample.\\
\noindent\textbf{Misclassification Awareness (MA)}: Wang \etal \cite{wang2019improving} imperially show that adversarial examples from misclassified samples result in a drop in robustness and propose extending the standard training loss with a misclassification aware regularization (line 7 in \tabref{tab:all}). However, this has the same drawback as LDS approaches. In this work, we provide a deeper understanding of the drawback of the current practice in generating adversarial examples from misclassified samples and propose an easier fix that dresses the root cause. 

\begin{sidewaystable}
\sidewaystablefn%
\footnotesize
\begin{center}
\begin{minipage}{\textheight}
\caption{Optimization objectives for AT methods and their main characteristics.}\label{tab2}

\begin{tabular*}{\textheight}{@{\hspace{0\tabcolsep}}r@{\hspace{0.5\tabcolsep}}l@{\hspace{0\tabcolsep}}c@{\hspace{0.5\tabcolsep}}c@{\hspace{1\tabcolsep}}c@{\hspace{1\tabcolsep}}c@{\hspace{1\tabcolsep}}c}
\toprule%
&Method & Training Objective & MM & LSR & MA & LDS  \\
\midrule
1&PGD-AT \cite{madry2017towards} & $\min_{\theta} \mathbb{E} \left[ \ell\left(\max_{\delta\in \Delta} f_{\theta} (x+\delta), y \right) \right] $ & - & - & - & -\\
2&TRADES \cite{miyato2018virtual} & $ \min_{\theta} \mathbb{E}\left[ \ell \left( f_{\theta}(x),y\right) +\lambda\times\ell_{CE}\left(\max_{\delta\in \Delta} f_{\theta} (x+\delta), f_{\theta} (x)\right) \right]$ & - & - & - & \checkmark \\
3&VILLA \cite{villa} & $\min_{\theta} \mathbb{E}\left[ \ell\left(f_{\theta}(x),y\right) + \lambda_1\times \ell(\max_{\delta\in \Delta}  f_{\theta} (x+\delta),y) +\lambda_2\times \ell_{KL}\left(\max_{\delta\in \Delta}  f_{\theta} (x+\delta), f_{\theta} (x)\right) \right] $ & - & - & - & \checkmark \\
4&MMA \cite{ding2018mma}&$\min_{\theta} \mathbb{E}\left[ \ell \left( f_{\theta}(x),y\right)\mathds{1}(f_\theta(x)\neq y)+ \ell \left(\max_{\delta\in \Delta}  f_{\theta}(x+\delta),y\right)\mathds{1}(f_\theta(x)=y) \right]$ & \checkmark & - & - & - \\
5&IAAT \cite{balaji2019instance}& $\min_{\theta} \mathbb{E}\left[ \ell \left(\max_{\delta\in \Delta}  f_{\theta}(x+\delta),y_i\right) \right] $ & \checkmark& - & - & -\\
6&CAT \cite{cheng2020cat} & $\min_{\theta} \mathbb{E}\left[ \ell \left(\max_{\delta\in \epsilon_i}  f_{\theta}(x+\delta),(1-c\epsilon_i)y_i+c\epsilon_i\text{Dirichlet}(\beta)\right) + \left( \max_{j\neq y_0}[Z(x_i')]_j-[Z(x_i')]_{y_0} \right) \right] $ & \checkmark& \checkmark & - & -\\
7&MART \cite{wang2019improving} & $\min_{\theta} \mathbb{E}\left[ \ell \left(\max_{\delta\in \Delta} f_{\theta}(x+\delta),y\right) + \ell_{KL}\left( f_{\theta} (x+\delta), f_{\theta} (x)\right) (1-[f_{\theta} (x+\delta)]_y)\right] $ & - & - & \checkmark &  \checkmark \\
8&A3T & $\min_{\theta} \mathbb{E} \left[ \ell\left(f_{\theta}\left(x+\argmax_{\delta \in \Delta} \ell\left(f_{\theta}\left(x+\delta\right), \hat{y} \right) \right),y \right) \right] $ & - & - & \checkmark &  - \\
8& $A3T^+$ & $\min_{\theta} \mathbb{E}\left[ \ell \left(f_{\theta}\left(x+\argmax_{\delta \in \Delta} \ell\left(f_{\theta}\left(x+\delta\right), \hat{y} \right) \right),(1-c\epsilon_i)y_i+c\epsilon_i\text{Dirichlet}(\beta)\right) + \left( \max_{j\neq y_0}[Z(x+\delta)]_j-[Z(x+\delta)]_{y_0} \right) \right] $ & \checkmark & \checkmark& \checkmark & -\\
\botrule
\label{tab:all}
\end{tabular*}
\end{minipage}
\end{center}
\end{sidewaystable}

\section{Accuracy Aware Adversarial Training (A3T)}
\label{sec:approach}
Given a data set $\mathcal{D} = \{(x_i,y_i)\}$, we consider a $K$- class classification setup. The classifier is
a $\theta$-parameterized score function $f_{\theta}(x) = (f^{1}_{\theta}(x),\ldots,f^{K}_{\theta}(x))$, where the
$i$-th class is assigned the score $f^{i}_{\theta}(x)$. The predicted label of $x$ is denoted as $\hat{y} = \argmax_{i}f^{i}_{\theta}(x)$. The classic adversarial training problem is formulated as a min-max optimization problem~\cite{madry2017towards}:
\begin{equation}
\min_{\bm{\theta}}\sum\limits_{(x,y) \in D}\big[ \max_{\delta \in \Delta} \ell(f_{\bm{\theta}}(x+ \delta), y )  \big],
\label{eq:adversarial_loss1}
\end{equation}
Here $\ell$ is the loss function used for training the classifier, and $\Delta$ is typically a $L_{\infty}$ norm ball bounded by $\epsilon$. 

For linear binary classification problems, the inner maximization can be solved analytically. Specifically, suppose $f_{\bm{\theta}}(\mathbf{x})=\bm{\theta}\mathbf{x} + b$
with class labels $y \in \{+1, -1\}$ and logistic loss 
\[
\ell(f_{\bm{\theta}}(x), y ) = \log(1 + \exp(-y \cdot f_{\bm{\theta}}(\mathbf{x})))
\equiv L(y.f_{\bm{\theta}}(x)).
\]
Then, 
\[
\delta^{\ast} =  \argmax_{\delta \in \Delta} \ell(f_{\bm{\theta}}(x+\delta), y ) 
\]
For nonlinear models (\eg deepnets), the inner maximization has to be
solved in an iterative fashion using PGD or its  variants~\cite{madry2017towards}.
In case of correctly classified training samples, the optimal perturbation would result in misclassified adversarial examples that are closer to the decision boundary than the original ones. However, in case of misclassified samples, the same optimal perturbation will result in adversarial examples that are further away from the decision boundary than the original samples. For high-capacity models, AT with such samples will result in an extremely non-smooth boundary with poor generalization.

\noindent\textbf{A3T Loss}: To prevent overfitting, A3T aims at generating close to decision boundary adversarial examples. To achieve this, it generates adversarial examples differently for classified and misclassified samples, \ie it solves different inner maximization objectives depending on the samples prediction accuracy:

\begin{equation}
    \begin{split}
        \min_{\theta} \left[\sum_{(x,y)\in\mathcal{D_{\theta}^+}} \ell\left(f_{\theta}\left(x+\argmax_{\delta \in \Delta} \ell\left(f_{\theta}\left(x+\delta\right), y \right) \right),y \right) \right. \\
        + \left.\sum_{(x,y)\in\mathcal{D_{\theta}^-}}\ell\left(f_{\theta}\left(x+\argmax_{\delta \in \Delta} \ell\left(f_{\theta}\left(x+\delta\right), \hat{y} \right) \right),y \right)\right].
    \end{split}
\label{eq:our_loss}
\end{equation}

Here, $\mathcal{D_{\theta}^+}$ and $\mathcal{D_{\theta}^-}$ are the set of correctly and misclassified examples, respectively. Note that the loss terms of the two inner maximizations have different arguments, \ie $y$ vs $\hat{y}$.
In particular, for the misclassified examples, the loss uses the
predicted label $\hat{y}$ as its second argument. 
The optimal perturbation $\delta^{*}$ can be computed using PGD as follows:
\begin{equation}
\delta^* :=  
\begin{cases}
\Pi_{\Delta} \left( \delta + \alpha \nabla_{x}\ell(f_\theta(x+\delta),y) \right) & \text{if } (x,y)\in \mathcal{D_{\theta}^+}\\
\Pi_{\Delta} \left( \delta + \alpha \nabla_{x}\ell(f_\theta(x+\delta),\hat{y}) \right)  & \text{otherwise} 
\end{cases}.
\label{eq:delta_a3t}
\end{equation}

Since $\hat{y}=y$ holds for samples in $\mathcal{D_{\theta}^+}$, \equref{eq:delta_a3t} can be simplified to
\begin{equation}
\delta^* := \Pi_{\Delta} \left( \delta + \alpha \nabla_{x}\ell\left(f_\theta(x+\delta),\hat{y}\right) \right).
\label{eq:delta}
\end{equation}
Combining \equref{eq:our_loss} and \equref{eq:delta} results in an alternative training objective to \equref{eq:our_loss}, \ie
\begin{equation}
        \min_{\theta} \left[\sum_{(x,y)\in\mathcal{D}} \ell\left(f_{\theta}(x+ \delta^*),y \right) \right].
\label{eq:final_objective}
\end{equation}
\begin{theorem} Let $f_{\bm{\theta}}(\mathbf{x})=\bm{\theta}\mathbf{x} + b$ be a linear model trained with a logistic loss $\ell$. Assume that $(\mathbf{x}_i,y_i)$ is a misclassified training example and that $\bm{\delta}_1 = \argmax_{\delta \in \Delta}\ell(f_{\bm{\theta}}(\mathbf{x_i}), y_i)$ and $\delta_2 = \argmax_{\delta \in \Delta}\ell(f_{\bm{\theta}}(\mathbf{x_i}), (1 - 2y_i))$ are the solutions to the inner maximization using standard AT (\equref{eq:adversarial_loss1}) and A3T (\equref{eq:our_loss}), respectively. We prove that
\[
\abs{\bm{\theta}(\mathbf{x}_i + \bm{\delta}_{2}) + b} \leq 
\abs{\bm{\theta}(\mathbf{x}_i + \bm{\delta}_{1}) + b}.
\]
\label{theorem:theorem1}
\end{theorem}
\begin{proof}
See Appendix.
\end{proof}
\noindent
Intuitively, \theoremref{theorem:theorem1} sates that adversarial examples from misclassified samples generated using A3T are closer to the decision boundary than the ones generated using AT. 

\noindent
{\bf Example 1:} Consider a 2-d binary linear classification setup, where the decision boundary is given by $ x_1 +  x_2 = 1$ and labels are $y=\pm 1$. The classifier weights are therefore $ (\theta_1,\theta_2) = (1,1)$. The adversarial perturbation $\bm{\delta} = -y\epsilon\sgn(\bm{\theta})$. Suppose a data point 
$(x_1,x_2)$ with label +1 is misclassified as -1. Then the traditional perturbation based on the original label is $(x_1 - \epsilon, x_2 - \epsilon )$ but based on the proposed approach will be $(x_1 + \epsilon,x_2 + \epsilon)$. Thus the proposed approach will
push the adversarial sample corresponding to $(x_1,x_2)$ towards the boundary.

\noindent
{\bf Example 2:} Consider the linear classification setup presented in \figref{fig:linearModel}. The standard training decision boundary corresponds to the black line. The AT decision boundary is shown in orange. A3T results in the green decision boundary. Note that A3T boundary is closer to the standard boundary than the AT one. This is due the adversarial examples from misclassified samples $A$ and $B$ being placed closer to the boundary in case of A3T. This also implies that A3T performance on clean test samples is closer to the original classifier.

\begin{figure}
    \centering
    \includegraphics[width=0.7\textwidth, trim = 1.5cm 2cm 1.5cm 2cm, clip]{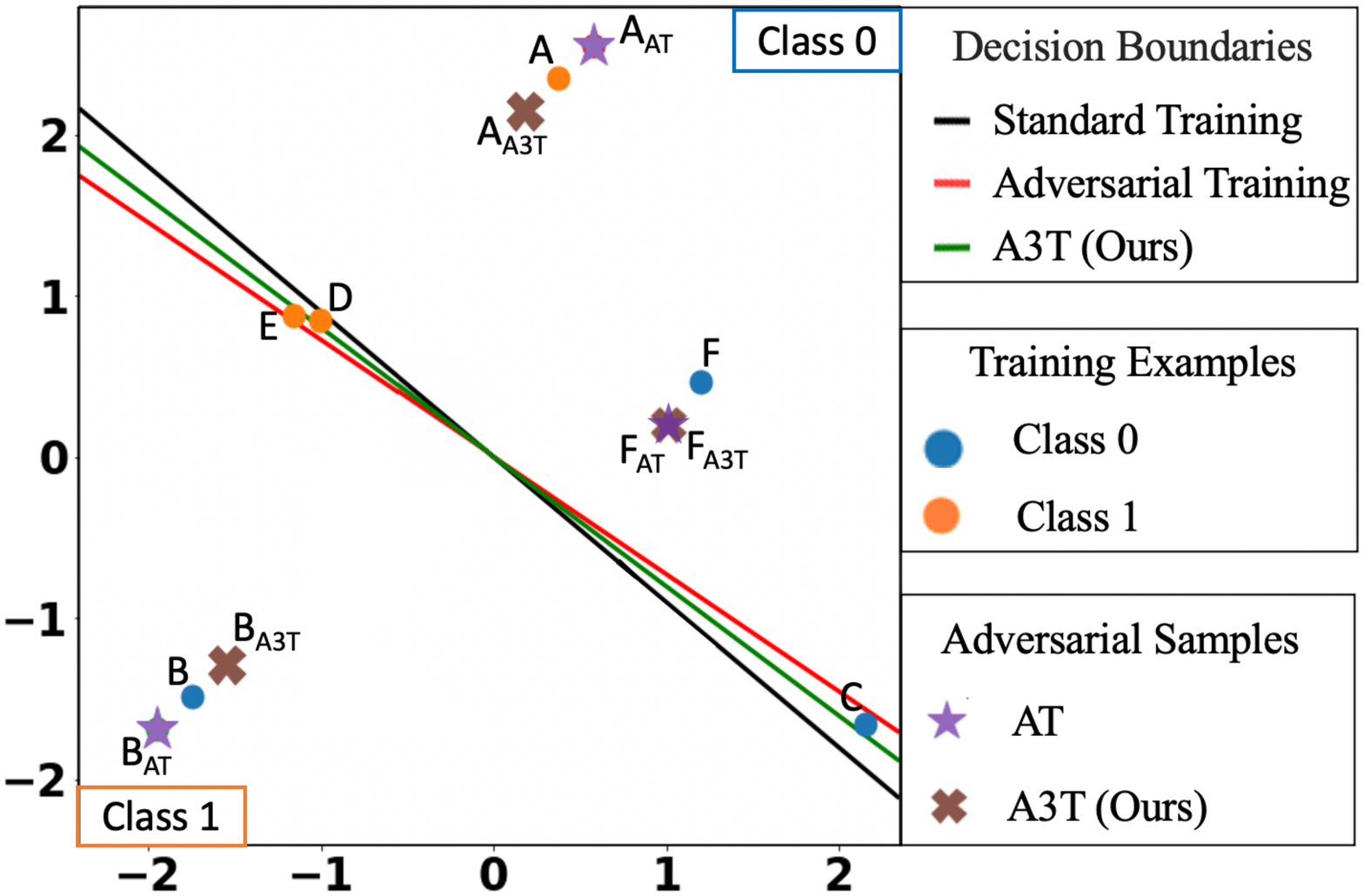}
    \caption{Adversarial training of of a linear classifier using standard AT and A3T on a two-dimensional synthetic dataset. The black, red, and green lines represent the decision boundaries of the model trained with standard training, AT and A3T, respectively. Samples $A$ and $B$ are misclassified by all the models. While A3T places adversarial examples from $A$ and $B$, \ie $A_{A3T}$ and $B_{A3T}$), closer to the boundary than $A_{AT}$ and $B_{AT}$ which are generated using AT. We omit correctly classified training examples for better visualization. }
    \label{fig:linearModel}
\end{figure}
\noindent\textbf{A3T algorithm}: For completeness, the full algorithm is shown in \algoref{a1:a3t}.

\begin{algorithm}[!tb]
\caption{Algorithm of A3T}\label{a1:a3t}
\begin{algorithmic}[1]
\Require $E$: the number of epochs, $D=\{(x_i, y_i)\}_{i=1}^n$:  the dataset, $f_{\bm{\theta}}(x)$: the machine learning model parametrized by $\bm{\theta}$, $\delta$: the perturbation initialized by $\sigma$ and limited by $\epsilon$, $\tau$: the global learning rate, $\alpha$: the adversarial learning rate, $S$: the number of PGD step, $\Pi$ the projection function.
\Ensure Model parameters $\theta$

\Procedure{A3T}{}
    \For{$e=1,..,E$}
        \For{$(x_i,y_i) \in \mathcal{D}$}
            \State $\hat{y_i} \leftarrow \argmax_{j \in \{1\ldots K\}} f^{j}_{\bm{\theta}}(x_i)$
            \State $\delta_i \sim \mathcal{N}(0,\sigma^{2}I)$
            \For{$s=1,..,S$}
                \State $\delta_i \leftarrow \Pi_{\epsilon_i}\left(\delta_i + \alpha \cdot \sgn\left(\nabla_{x_i}\ell(f_\theta(x_i+\delta_i),\hat{y_i})\right)\right)$

            \EndFor
            \State $\theta \leftarrow \theta - \tau \nabla_\theta \ell(f_\theta(x_i+\delta_i),y_i)$ 
        \EndFor
    \EndFor
\EndProcedure

\end{algorithmic}
\end{algorithm}

Note that A3T, Label Smoothing Regularization (LSR) and Margin-Maximization (MM) are complementary. LSR addresses a different cause of overfitting as explained in \sectionref{sec:related_work}. MM is orthogonal as it searches for the optimal perturbation set $\Delta$ instead of the optimal perturbation $delta$ in a predefined set $\Delta$. We denote by A3T$^+$ the approach that combines A3T, LSR and MM. A3T$^+$ objective and algorithm are given in Line 8 of Table \ref{tab:all} and \algoref{a1:a3t+}, respectively.

\subsection{Improvement over Prior Work}
So far, two misclassification aware adversarial training approaches have been proposed~\cite{wang2019improving,ding2018mma}. Wang et al. \cite{wang2019improving} were the first to explicitly recognize that misclassified and correctly classified samples should be treated differently and proposed a misclassificiation aware adversarial training method (MART). 
Their training objective consists of two separate loss terms as given in line 7 of Table \ref{tab:all}.
The first term corresponds to the adversarial loss used in conventional AT, and does not distinguish between correct and misclassified examples. 
The second term is a KL divergence loss weighted by $(1-f_{\bm{\theta}}(\mathbf{x}))$ and introduced by \cite{miyato2018virtual} effectively lowering the impact
of correctly classified samples and encouraging label smoothness around
misclassified examples.
In the context of Figure~\ref{fig:a3t}, MART populates the space around the misclassified examples by new samples of both the true class and the incorrectly predicted class. 
The true class samples will be located in a direction that is maximally away from the decision boundary and will serve as hard examples for the model. 
In contrast, the regularization based label smoothing will add samples of the incorrect class isotropically. 
In effect, the latter group of samples neutralizes the impact of the hard examples created by the conventional AT method and prevents the model from exhibiting even worse overfitting behavior.

Ding et. al.~\cite{ding2018mma} use margin-maximization to propose an implicit treatment of misclassified examples.
The perturbation margin of each example is individually determined through increasing perturbation strength gradually until an adversarial sample is generated. 
Since misclassified examples are adversarial in nature, they yield the lowest possible margin. Thus effectively adversarial samples created for misclassified examples are not significantly perturbed.

Our approach, A3T, fundamentally differs from the above methods in that the adversarial samples created include easy samples that are placed towards the decision boundary, as opposed to creating no samples or samples that are further away from the boundary. These easy samples essentially serve to correct for any potential deformation in that part of the decision boundary. 
In other words, we attribute the cause of a misclassified example to the presence of opposing class examples in the immediate vicinity of that example, preventing the model from successfully fitting them. 
By creating samples that yield lower loss, our method reduces the effect of misclassified examples on the resulting decision boundary.

 \section{Experiments}
\label{sec:experiments}

\subsection{Results on Synthetic Data}
\label{sec:synth}
To better observe the behavior of A3T, we evaluate it on a synthetic dataset introduced in \cite{villa}.
This dataset contains 1,016 samples from two classes in a two-dimensional space over Real numbers.
Samples in each class are created on two trajectories each following the shape of a crescent moon.
Each data point is then projected onto a 100-dimensional vector space.
For classification, a model with 100 neurons is trained using only 16 randomly selected training samples per class. With such a small training sample size the model is prone to overfitting, thereby allowing better evaluation of the generalization capability.
The model is trained for 400 epochs at a learning rate of 0.01.
For the first 100 epochs, the standard training loss is used and in the remaining epochs adversarial training loss is incorporated. 
Adversarial samples are created through PGD attack, Eq. \ref{eq:PGD}, after applying a random perturbation (normally distributed with zero mean and standard deviation of 0.05) with only five PGD steps assuming $\alpha=0.1$ and $\Delta=0.4$.

To verify our intuition on the operation of A3T (Fig \ref{fig:a3t}) we first show the resulting decision boundary considering this binary classification task. 
As part of the tests, 20 models are generated for each training approach using the same set of training samples and the fixed parameter values. 
Following the standard training (for 100 epochs) all models were found to correctly classify all but one (out of 32) training samples, overall exhibiting a good fitting capability.
The decision boundaries are then drawn based on the positions of all data points classified with the least confidence, i.e., by averaging the locations of all training samples in all runs that were classified with confidence in the $0.49-0.51$ probability range.
Figure \ref{fig:points} shows the decision boundaries obtained after the first 100 epochs and when the training is completed, respectively, as gray and black lines for both AT and A3T.
All adversarial samples created beyond epoch 100 are also displayed as orange and turquoise dots on the figure. 

\begin{figure}
\centering
  \begin{minipage}[a]{0.48\columnwidth}
    \centering
    \includegraphics[width=1\columnwidth]{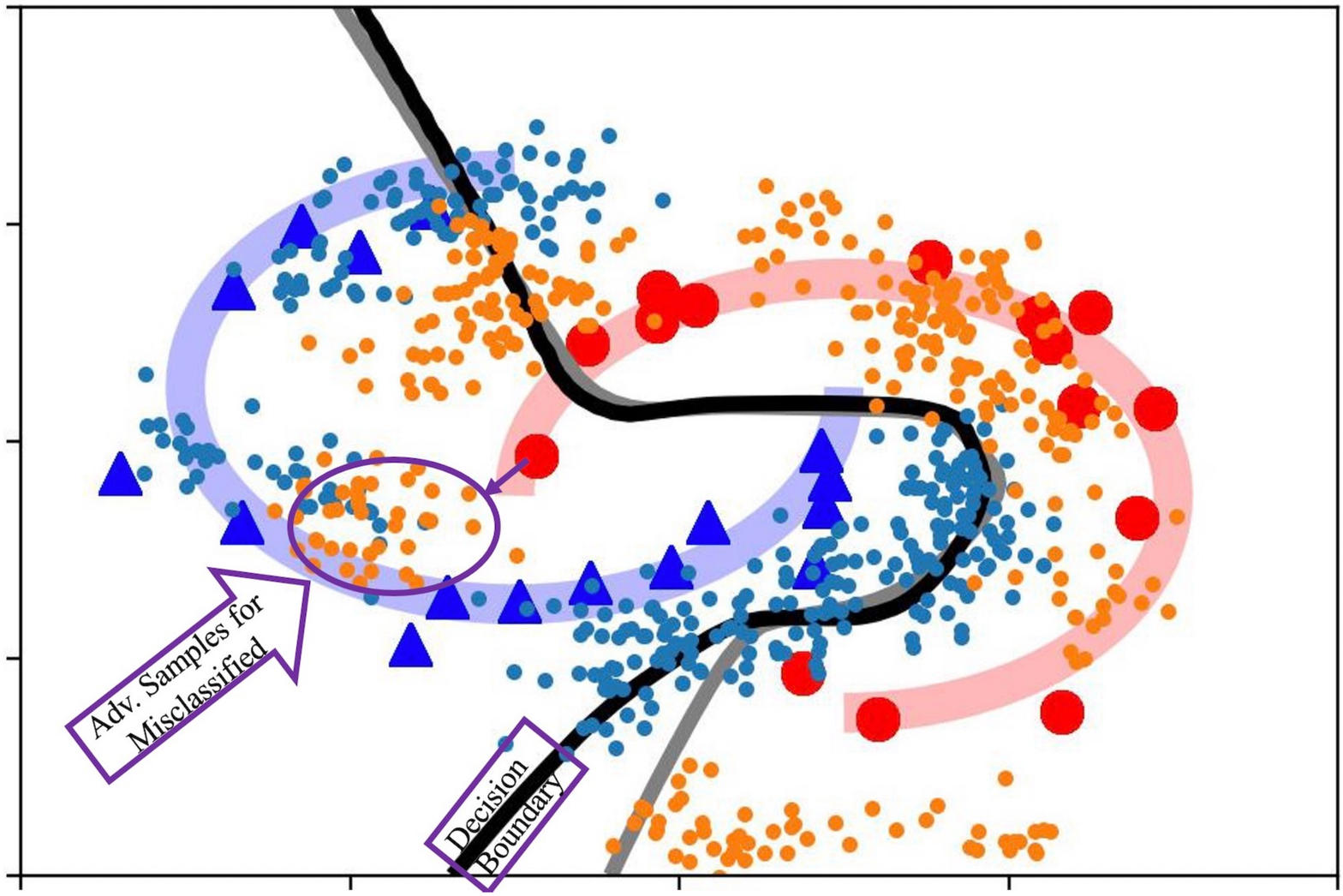}
    \centerline{(a) AT}
  \end{minipage}
  \begin{minipage}[d]{0.48\columnwidth}
    \centering
    \includegraphics[width=1\columnwidth]{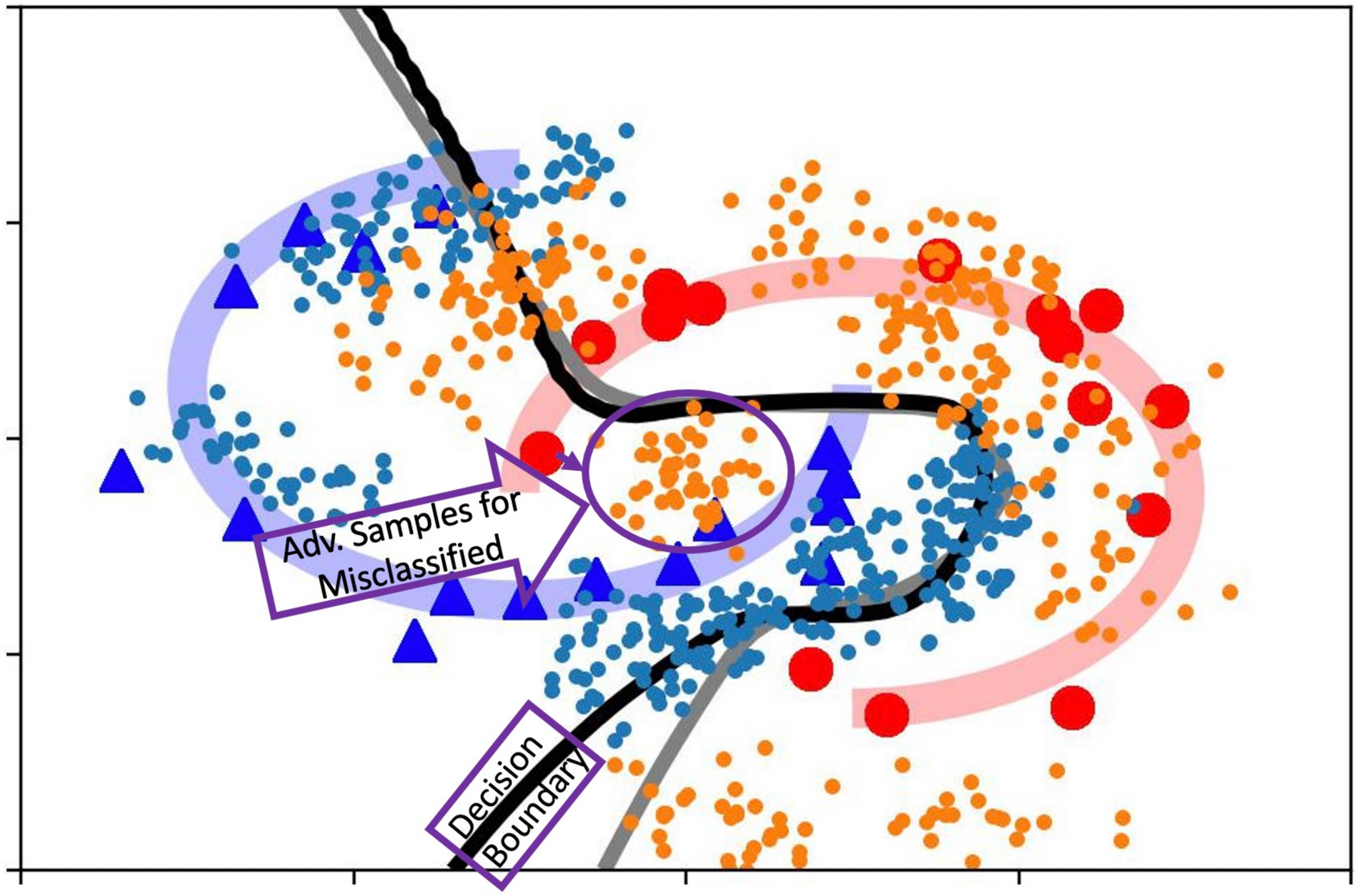}
    \centerline{(b) A3T}
  \end{minipage}
   \caption{The orange and turquoise dots show adversarial samples created by AT and A3T methods. For each method, decision boundaries are determined based on the position of test points that are assigned to the correct class with a probability of around 0.5. The gray line is obtained by the standard training (after epoch 100 is fixed) and the black line corresponds to adversarial training (obtained after epoch 400).
  A3T pushes adversarial samples misclassified example (orange colored sample) closer to the decision boundary.
   }
  \label{fig:points}
\end{figure} 

The difference between the conventional AT approach and A3T can best be seen around the misclassified examples of the red class, at the upper end of the crescent.
In the case of A3T, the adversarial samples, i.e., orange points around the misclassified examples, are created between the misclassified example and the decision boundary. 
Whereas for AT, adversarial samples are created further away from the boundary.
This further resulted in intermixing of red-class adversarial samples with the blue-class data points, which may force the model to learn an extremely non-smooth decision boundary.
In other respects, it can be seen that the final decision boundary (black line) for AT and A3T are very similar and therefore the robustness behavior may be expected to be on par with each other.

To further examine the generalization behavior, we repeated the same experiment by randomly selecting the training points for each run and averaging results over 50 runs. 
Figure \ref{fig:syntCounter} shows all training points used as well as the corresponding decision boundaries. 
(Note that training examples remain the same for both AT and A3T within a run, but they change across runs.)
Here, the generalization capability of A3T is more visible as it is able to correctly classify most of the blue-class points.
In the case of red-class data points, both methods perform similarly. 
These observations are also reflected in the measured classification accuracies where A3T achieved 83.6\%, i.e., 1.2\% better than AT.

\begin{figure}
\centering
  \begin{minipage}[a]{0.48\columnwidth}
    \centering
    \includegraphics[width=1\columnwidth]{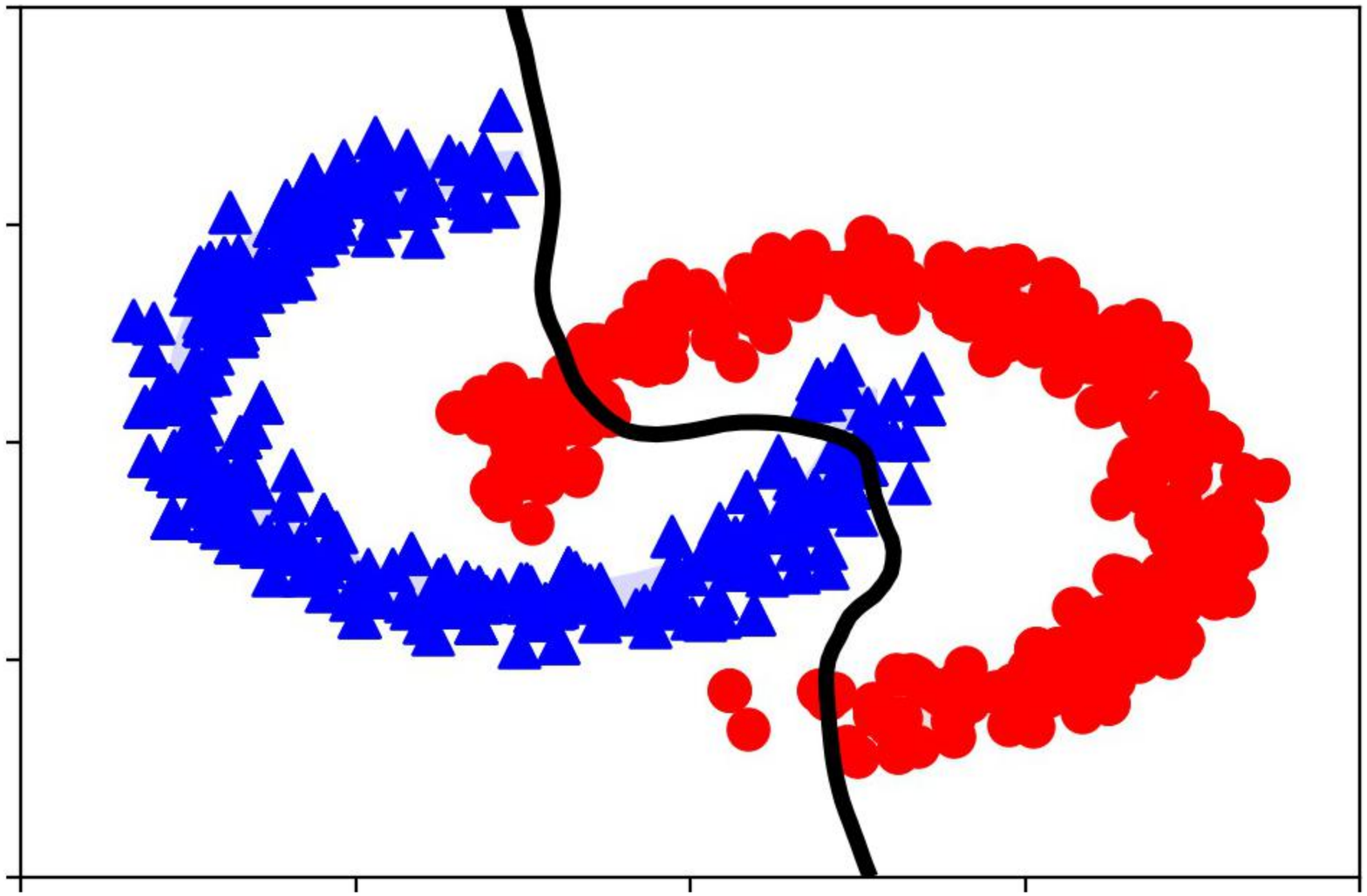}
    \centerline{(a) AT (Acc: 82.4\%) }
  \end{minipage}
  \begin{minipage}[d]{0.48\columnwidth}
    \centering
    \includegraphics[width=1\columnwidth]{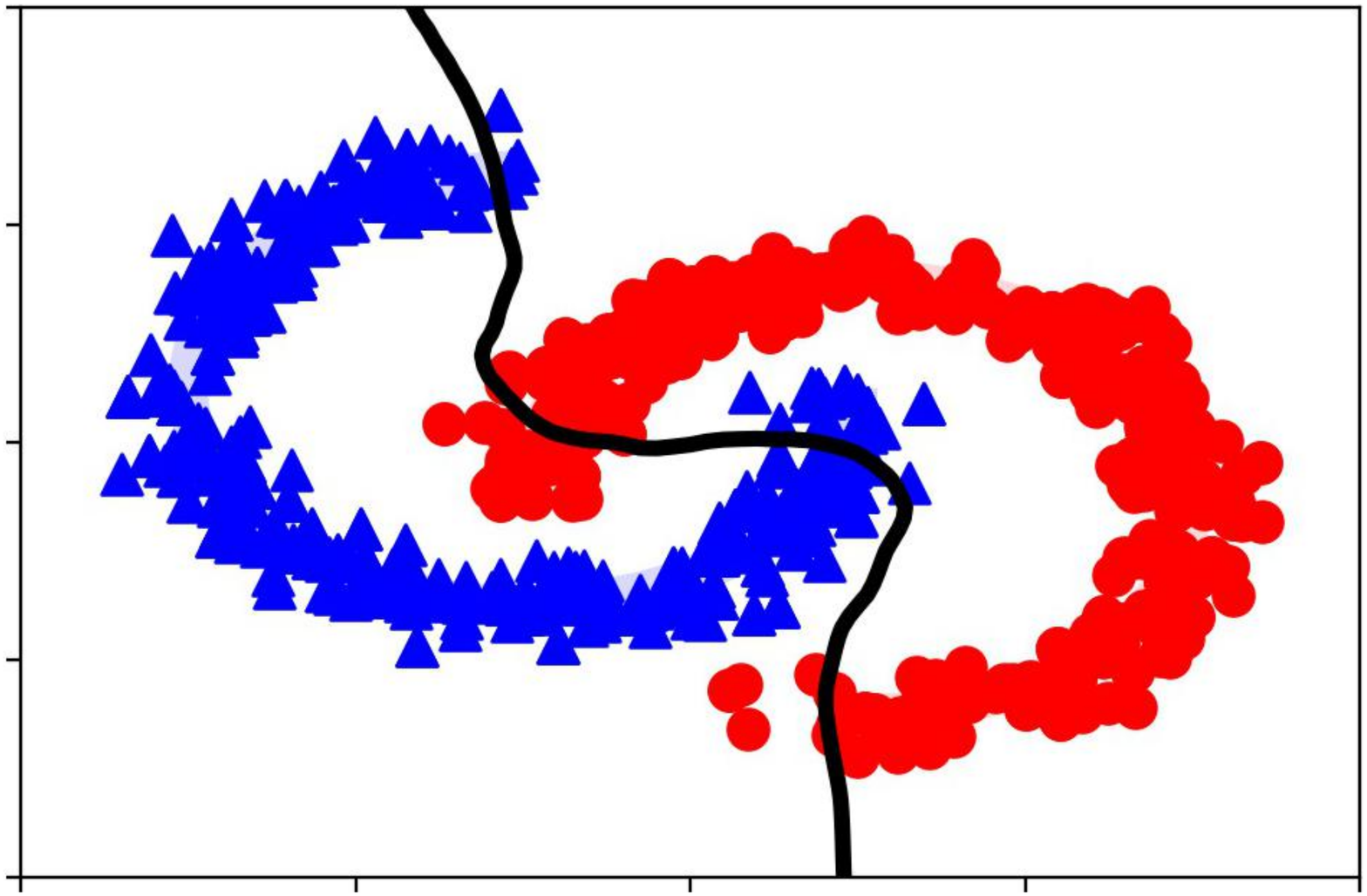}
    \centerline{(d) A3T (Acc: 83.6\%) }
  \end{minipage}
  \caption{Resulting decision boundaries obtained after 50 runs when training samples are picked at random for each run. 
  The improved generalization capability of A3T is reflected in the measured accuracy and can also be seen in the way it can more correctly classify test samples compared to AT.}
  \label{fig:syntCounter}
\end{figure}

\subsection{Results on Real Data}
We now assess the robustness and generalization trade-off yielded by A3T on tasks related to vision, natural language processing, and tabular data. 
Models obtained by standard training and adversarial training are evaluated under attack-free and attack circumstances.
The models are attacked using adversarial samples produced by PGD \cite{madry2017towards} and AutoAttack \cite{croce2020robustbench}. 

\subsubsection{Vision Tasks}

Most adversarial training methods are proposed to improve the robustness of image classifiers. 
The effectiveness of these methods has been evaluated on CIFAR-10 datasets using fixed architectures such as WideResNet-34-10 and 
WideResNet-28-10 with varying attack and training parameters. Therefore, in our tests, we also consider a similar setting.

MART is the only other method that explicitly formulates a treatment for misclassified examples during adversarial training. 
Therefore, we start by comparing A3T against MART on how they cope with overfitting using the standard cross-entropy loss.
In their paper, MART uses boosted cross-entropy loss, and report a significant improvement in accuracy  ($\sim$3\%) compared to the use of standard cross entropy loss\footnote{This is demonstrated in Fig. (2b) of \cite{wang2019improving}}. 
Since boosted loss can be generically incorporated to all adversarial training approaches, we only considered the use of cross-entropy loss to ensure a comparison on an equal footing.
To be in line with the previous methods, we trained all models for a total of 100 epochs using SGD with momentum $0.9$, weight decay $2\times10^{-4}$, an initial learning rate of $0.1$, and we decay the learning rate by 90\% at the 75th, 90th epoch.
Table \ref{tab:martvsa3t} provides corresponding accuracy results for standard and under-attack settings. 
As can be seen, the most notable difference concerns the natural accuracy values, where A3T performed 4.7\% better than MART. 
This can be attributed to A3T's ability to better mitigate overfitting, which mainly reveals itself in the model's generalization ability.
However, this gain comes at the cost of reduced robustness, where MART is seen to perform 2-3\% better. 
When compared to conventional AT, A3T improves both natural and robust accuracy by around 2\%.

\begin{table}[!ht]
	\caption{Comparison of the two misclassification-aware adversarial training methods based on results obtained on the CIFAR-10 dataset using the WideResNet-34-10 model\tablefootnote{The presented results for MART are obtained using the public implementation shared by its authors.}.A3T results in better generalization than MART at near-similar levels of robutness.} 

	\centering
    	\begin{tabular}{cccccc}
    	\toprule
    	\multirow[b]{2}{*}{Defense}&\multirow[b]{2}{*}{Natural}  &  \multicolumn{2}{c}{FGSM}    &  \multicolumn{2}{c}{$PGD^{20}$} \\ \cmidrule(lr){3-4} \cmidrule(lr){5-6}
    	                        &                          & Best     & Last               & Best     & Last \\ \hline
    	AT \cite{madry2017towards}    &  87.30 &  56.10   & 56.10 &   52.68   & 49.31\\ 
    	$MART_{CE}$ &  84.24 &  66.86   & 65.28 &   \textbf{56.27}   & 52.22 \\
    	$A3T$  &  \textbf{88.98} &  \textbf{68.04}   & \textbf{67.12} &   54.59   & \textbf{52.39}\\
    	\bottomrule
    	\end{tabular}
	\label{tab:martvsa3t}
\end{table}

Next, we compare A3T with several adversarial training methods that leverage the ideas of MM \citep{ding2018mma,cheng2020cat}, LSR \citep{balaji2019instance,wang2019bilateral,cheng2020cat}, and LDS \cite{zhang2019theoretically} as well as the A3T$^+$ method that incorporates MM and LSR approaches \cite{cheng2020cat}.
Here we determined that the results reported by these methods include some inconsistencies, making a fair comparison challenging. 
These include the learning-rate schedule used by a method that affects the proneness to robust overfitting.
(That is, schedules that include a larger number of training epochs are expected to report lower final accuracy values.)
Another factor relates to the PGD step size used during attacks which determines the strength of the attack applied to a model.
To circumvent these ambiguities, we decided to use the robustness results reported by the AutoAttack (AA) benchmark\footnote{https://robustbench.github.io/\#leaderboard} and evaluated A3T and A3T$^+$ accordingly.
Table \ref{tab:cmp} provides corresponding results for several adversarial training methods. 

\begin{table}[htbp]
\caption{The clean and AutoAttack (AA) \cite{wang2021convergence} accuracy values of adversarially trained WideResNet models.}
\label{tab:cmp}
\begin{center}
\begin{tabular}{l|c|c|c|c|c}
\toprule
Methods                                       & Natural Acc      & AA Acc      &  Avg. Acc   & Architecture \\
\hline
Natural training                              &      \textbf{95.93}     &       0     &  47.97       &  WideResNet-34-10   \\
AT \citep{madry2017towards}                   &      87.14     &    44.04    &  65.59      &  WideResNet-34-10   \\
MMA \citep{ding2018mma}                       &      84.36     &    41.44    &  62.9       &  WideResNet-28-4 \\
TRADES \citep{zhang2019theoretically}         &      84.92     &    53.08    &  69.0       &  WideResNet-34-10 \\
Bilateral AT \citep{wang2019bilateral}        &      \bf{92.80}     &    29.35    &  61.08      &  WideResNet-28-10 \\
MART \citep{wang2019improving}\tablefootnote{The accuracy values presented on the online leaderboard of the AA benchmark correspond to the semi-supervised version of MART trained with additional unlabeled data. For a fair comparison, here, we present results for the model trained in the supervised learning setting that is obtained from the project's GitHub repository.}
                                              & 83.62          &   \underline{55.69}    &    69.65   & WideResNet-34-10 \\
CAT  \citep{cheng2020cat}\tablefootnote{Upon examination of the public implementation of the CAT method, we noticed the authors have mistakenly reported the training robust accuracy instead of the test robust accuracy values. These values are obtained by rerunning their experiments.}
                                              &      86.16     &\textbf{55.79}&   70.97     &   WideResNet-34-10 \\
Dynamic AT \citep{wang2021convergence}        &      85.03     &    48.70    &   66.87     &   WideResNet-34-10 \\
A3T                                           &      \underline{88.98}     &    52.07    &  70.52     &  WideResNet-34-10 \\
$A3T^+$                                       &      \underline{87.80}     & \underline{55.59} & \textbf{71.69} &   WideResNet-34-10          \\
\bottomrule
\end{tabular}
\end{center}
\end{table}

These results demonstrate that A3T and A3T$^+$ yield the highest natural accuracy after the Bilateral AT method which exhibited very limited robustness under AA attack.
In terms of robust accuracy, CAT yields the best accuracy, performing only marginally better than A3T$^+$ (+0.2\%)
and noticeably better than A3T (+3.7\%).
However, both A3T and A3T$^+$ yield a higher natural accuracy compared to CAT (+2.8-1.6\%).
These results overall present a more granular view of the generalization vs. robustness trade-off
with MM and LSR allowing a shift in favor of robustness, while misclassification awareness and LDS tilting the balance more towards improved generalization. 
When the average accuracy is evaluated, it is seen that A3T$^+$ offers a better
trade-off between natural and AA accuracy as it allows an increase in the former while keeping the robust accuracy
on par with other methods.
This finding further strengthens the idea that MM, LSR, and misclassification awareness are addressing different aspects of adversarial training and that they are indeed compatible with each other.

\subsubsection{Natural Language Processing Tasks}

Tests are performed on the GLUE tasks \cite{wang2018glue}\footnote{Out of the eight defined tasks, we excluded the one related to the STS-B dataset as it involves a regression task.}.
This benchmark contains seven datasets for two sentiment analysis tasks, two similarity tasks and three inference tasks.
Due to the discrete nature of  text, generating adversarial samples in the input space is nontrivial  because it involves projecting the continuous perturbation computed in the embedding space to the input space. 
Although several input-domain approaches are proposed to create adversarial samples, they are not efficient \cite{altinisik2022impact}.
An alternative approach is to perform adversarial training in the latent space  without the need for creating adversarial inputs \cite{villa}.
In our experiments, we also adopt this approach both during adversarial training and when performing adversarial attacks. 
It must be noted that in the case of launching an adversarial attack this approach is impractical to implement, but it nevertheless corresponds to a worst-case attack setting and allows a better evaluation of the model's robustness.

For tests, we fine-tuned the deBERTa-base model \cite{he2021deberta} from the HuggingFace library.
For each dataset in the GLUE benchmark, we created a fine-tuned, task-specific model by first training the model for three epochs at the suggested learning rate of $2e^{-5}$ \cite{he2021deberta}.
Then, the first five layers are frozen and fine-tuning is continued for another three epochs with either standard or adversarial training.
In all cases, adversarial samples are generated using three-step PGD assuming $\Delta = 0.01$ using random initialization with zero mean and standard deviation of 0.005.

\small{
\begin{table*}[t]
    \caption{Adversarial Training Results on GLUE Benchmark. Results show that for all tasks, A3T performs on par with or better than the conventional AT.}
    
    \centering
    \begin{tabular}{c|ccc|ccc}
        \toprule
        \multirow{2}{*}{Dataset}	           & \multicolumn{3}{c|}{Attack-Free Tests}             & \multicolumn{3}{c}{Under Attack Tests} \\
                   &  Standard      &  AT     & A3T     & Standard      &  AT & A3T\\ \hline
        MRPC	&88.7 &	69.8 &	\textbf{88.9}           &	12.0 &	\textbf{65.6} &	63.5 \\ \hline  	
        QNLI	&92.1 &	\textbf{92.6} &	92.3            &	40.1 &	50.5 &	\textbf{55.8} \\ \hline  	
        QQP	    &91.4 &	91.4 &	91.4                    &	45.8 &	63.7 &	\textbf{68.3} \\ \hline  	
        RTE	    &73.3 &	71.9 &	\textbf{77.2}           &	1.2 &	\textbf{31.2} &	29.7 \\ \hline  	
        SST-2	&93.0 &	93.9 &	\textbf{94.0}           &	23.1 &	\textbf{50.4} &	50.3 \\ \hline  	
        MNLI	&86.1 &	\textbf{86.8} &	86.5            &	4.7 &	33.2 &	\textbf{43.7} \\ \hline  	
        \textbf{Avgerage Acc} & 87.4 &	84.4 &	\textbf{88.4} &	21.2 &	49.1 &	\textbf{51.9} \\ \hline  
        CoLA	&\textbf{61.7} &	37.0 &	37.7        &	-91.4 &	\textbf{-23.0} &	-29.0 \\ \hline  
        \textbf{Avgerage All}&\textbf{83.7} &	77.6 &	81.2 &	5.1 &	38.8 &	\textbf{40.3} \\
        \bottomrule
    \end{tabular}
    \label{tab:POSdist}
\end{table*}
}

Table \ref{tab:POSdist} reports classification performance yielded by various models under both normal and adversarial attack scenarios to evaluate the improvement of A3T over AT.
As compared to the baseline accuracy of the model, i.e., attack-free test setting, A3T yields a small drop in the average performance (-2.5\%) but performs noticeably better (+3.6\%) than AT as displayed in the last line of the table. 
Similarly, under an adversarial attack test setting, A3T yields an improvement of 1.5\% over conventional AT.
The improvement due to A3T is more noticeable when the results of CoLA task is removed from the average as it uses a different metric than accuracy (third to the last line). 
In that case, the use of A3T is found to result in a performance improvement of 2.8\% over AT.

\subsubsection{Tabular Tasks}

The robustness of machine learning models that make financial decisions against adversarial attacks is another important area of concern.
Since financial data is mostly in tabular format, we tested the effectiveness of A3T on tabular data for two classification tasks.
The first test involves Matlab's {\em Retail Credit Panel} dataset \cite{matlab}, and the task is to predict the overall yearly default rate for subjects.
Each dataset sample contains risk factor, year, and the default status (0 or 1) for a subject. 
To predict what percent of subjects default in a given year, a logistic regression model with two hidden layers, consisting of 256 and 128 neurons, is trained over 500 epochs, with a learning rate of 0.001.
The adversarial training is only applied between epochs 100 and 500 using the attack parameters obtained after a grid search\footnote{Adversarial attack parameter ranges for the grid search were as follows: initialization noise variance for PGD \{0.05, 0.1\}; PGD step size $\alpha$ \{0.2, 0.1, 0.05\}, and projected perturbation limit $\Delta$ \{0.2,0.4\}. The total number of PGD steps was fixed to five.}.

\begin{figure}
\centering
    \centering
    \includegraphics[width=.75\columnwidth, trim = 1.5cm 3.5cm 1.5cm 3.5cm, clip]{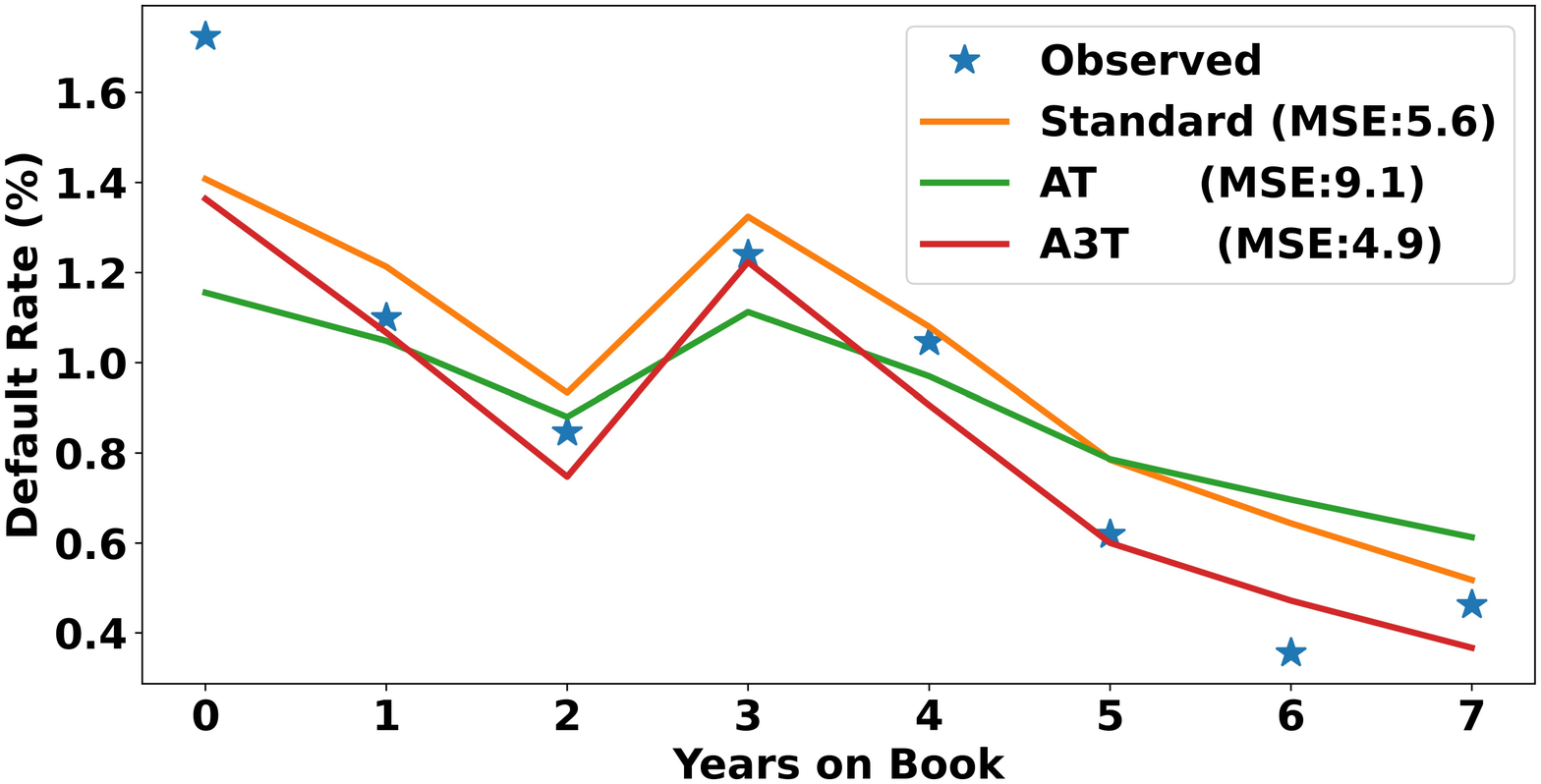}
   \caption{Yearly default rates predicted by robust and standard models obtained averaged over 50 runs for each model. MSE values are computed between model predictions and realized, ground-truth rates. Computed values are normalized by $e^{-5}$ for ease of viewing. }
  \label{fig:MatlabRes}
\end{figure}

Since the input data is very low-dimensional, no adversarial attack is performed. Instead, we investigated how adversarial training impacts the model's performance in the attack-free setting. 
Figure \ref{fig:MatlabRes} provides the average default rates predicted by different models after 50 runs in comparison to the ground truth rates. 
Results show that A3T predictions follow the realized default rates more closely than other approaches.
To better evaluate the fit of each model mean squared error (MSE) of predictions with respect to actual rates is also computed across all years.  
Accordingly, A3T is seen to yield the lowest MSE among all approaches, even outperforming the non-robust model.

The second task is binary classification and involves two datasets, namely, the European card dataset \textit{(ECD)} \cite{ECD} and the \textit{Adult} dataset \cite{adult}. 
The former dataset contains 492 fraud and 10K genuine transactions randomly downsampled from more than 280K transactions with 31 features, and the goal is to identify the type of transaction.
The latter involves 32K samples with nine features, and the objective is to predict whether the income of a subject is higher than \$50K or not. 
All categorical values are represented by one-hot encoding, and logistic regression models with the same network architecture described above is trained.
Rather than using a fixed adversarial training setting, multiple robust models with different training parameters are generated.  
Each model is then tested under 12 adversarial attack scenarios parameterized by the same possible parameter value configurations considered in the grid search,   
and the average of resulting accuracy values is taken as a model's under-attack prediction accuracy.
The standard and under-attack accuracy achievable by a model is finally determined by averaging corresponding values over five runs. 
Resulting accuracy values are plotted in Fig. \ref{fig:tbl}.

\begin{figure}
\centering
  \begin{minipage}[a]{0.48\columnwidth}
    \centering
    \includegraphics[width=1\columnwidth, trim = 3.4cm 10cm 2.5cm 10cm, clip]{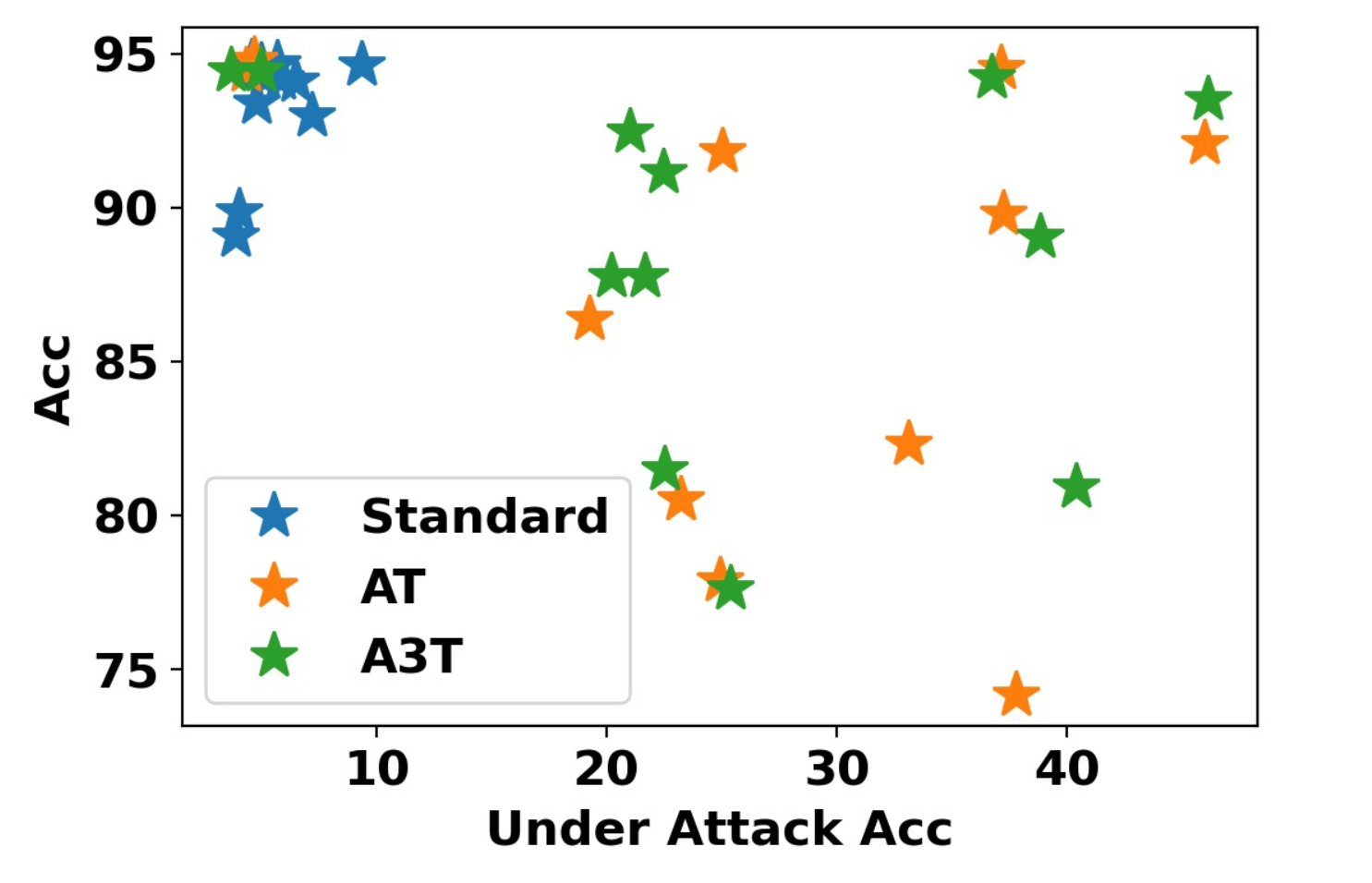}
    \centerline{(a) ECD}
  \end{minipage}
  \begin{minipage}[c]{0.48\columnwidth}
    \centering
    \includegraphics[width=1\columnwidth, trim = 3cm 10cm 2.5cm 10cm, clip]{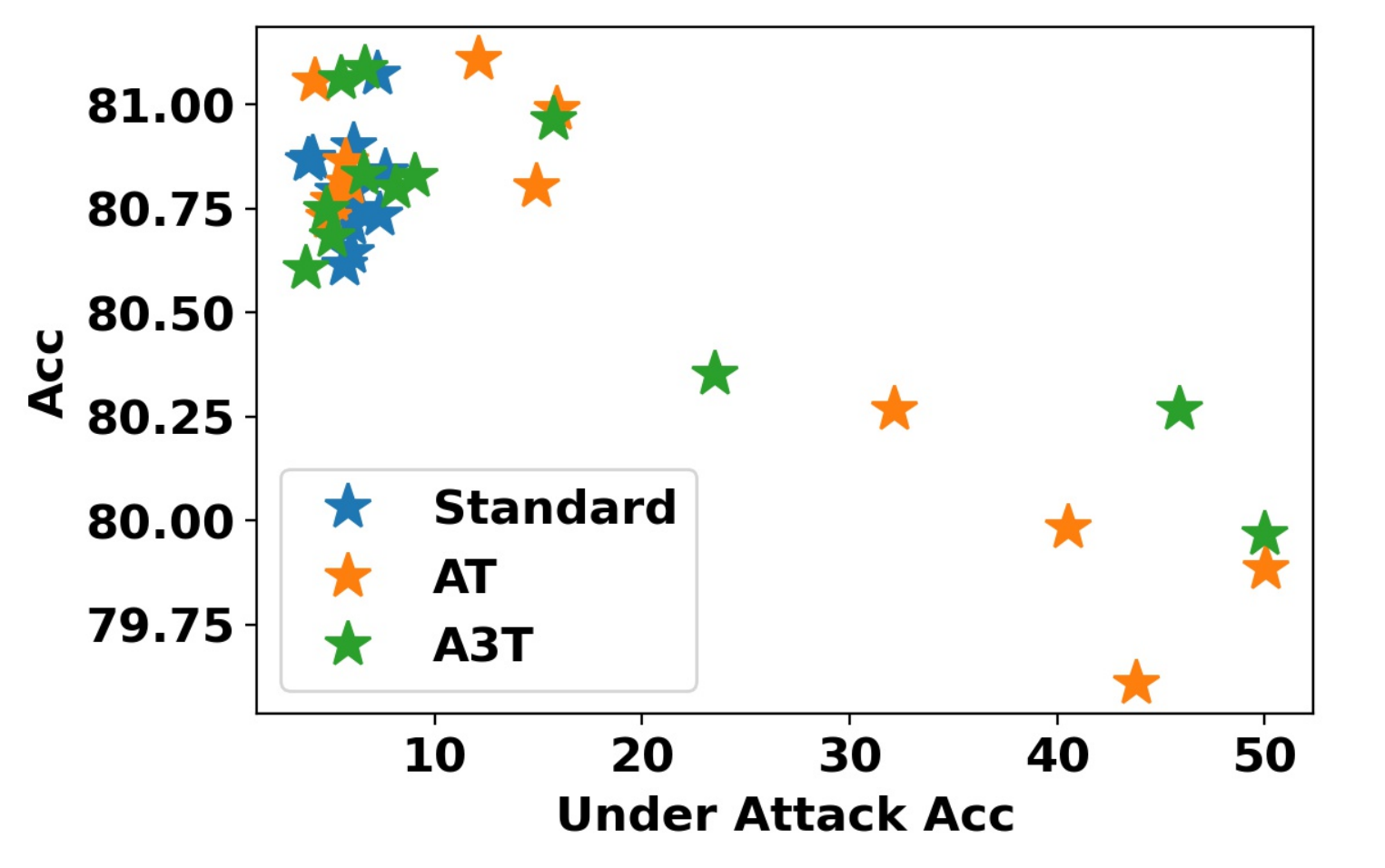}
    \centerline{(b) Adult}
  \end{minipage}
  \caption{Model accuracy values for both normal and under-attack settings obtained by averaging over five runs.  
   For both AT and A3T, 12 models are created for varying parameter values that govern the adversarial sample generation process. A3T is found to yield the most favorable performance, exhibiting high accuracy in both attack-free and under attack scenarios.}
  \label{fig:tbl}
\end{figure}

In the figure, the upper right corner corresponds to a performance regime where a model performs well both in attack-free and under-attack settings.
Thus, the best performing model can be identified based on how close it gets to that corner.
It can be seen that several robust versions of A3T exhibit high under-attack performance with only a slight drop in performance in the attack-free setting. 
This result also demonstrates the importance of the choice of adversarial training parameters on model accuracy.

\section{Discussion and Conclusions}
Adversarial training uses adversarial samples to retrain machine learning models in order to promote robustness. 
The robustness of a model can, however, only be attained at the expense of model generalization.
In this work, we propose a new adversarial training method that yields a more favorable robustness-generalization trade-off. We discuss below further points to provide a better insight about our proposed approach. 

\bmhead{Generation of non-adversarial samples}
The underlying idea of our misclassification-aware adversarial training approach (A3T) is to prevent the creation of hard examples from misclassified samples, thereby reducing the risk of overfitting. 
A3T effectively realizes this by generating samples that are non-adversarial in nature.
That is, instead of applying a perturbation that maximizes the loss, A3T computes a perturbation that minimizes the loss, defeating the purpose of an adversarial sample. 
Hence, the newly generated samples by A3T are maximally close to the decision boundary as opposed to adversarial samples that are maximally away from the boundary.
From this perspective, imposing a bound on the extent of loss reduction may not seem meaningful. 
However, since samples with arbitrarily low-loss values will constitute easy examples, they will likely be less informative for the model. 
Therefore, we also evaluated how much to reduce the loss of a misclassified sample. 
Our test results indicated that A3T's generalization ability does not improve when generated samples have a loss similar to that of the misclassified one.
Tests also indicated that a loss-reduction around $\Delta$ makes A3T most effective.

\bmhead{Incorporation of LDS with A3T}
Since MM, LSR, LDS, and MA are proposed to address different aspects of adversarial training, a question to be answered is if and to what extent these design choices interfere with each other.
In fact, our results with A3T$^+$ show that combining MM, LSR, and MA help achieve a more preferable generalization vs. robustness trade-off. 
However, incorporating LDS with A3T, i.e., using A3T during inner maximization and LDS as an additional regularization, will be ineffective. 
Crucially, LDS applies label smoothing by placing samples of predicted class, which for a misclassified sample includes opposing class samples, around the misclassified sample. 
In contrast, A3T generates samples of the same class as the misclassified sample towards the decision boundary. 
Hence, incorporating the two together will potentially contribute to the model's overfitting as samples of opposing classes will be placed in the same vicinity.

\bmhead{Influence on robust overfitting}
Robust overfitting is another artifact of adversarial training wherein a model test accuracy remains similar but the robust accuracy exhibits a drop as the number of training epochs increase \cite{dong2021exploring,rice2020overfitting}. 
Our observations show that A3T also suffers from robust overfitting. 
This is in agreement with the findings of MART, where a larger variation between best and last accuracies is observed compared to A3T.
Hence, we can deduce that the root cause for this phenomenon does not mainly relate to the treatment of misclassified examples.

\section{Acknowledgement}
This work is partially supported by the Qatar National Research Fund (QNRF) grant NPRP11C-1229-170007.





\bibliography{bibfile}

\newpage
\section{Appendix}
\begin{theorem*} Let $f_{\bm{\theta}}(\mathbf{x})=\bm{\theta}\mathbf{x} + b$ be a linear model trained with a logistic loss $\ell$. Assume that $(\mathbf{x}_i,y_i)$ is a misclassified training example and that $\bm{\delta}_1 = \argmax_{\delta \in \Delta}\ell(f_{\bm{\theta}}(\mathbf{x_i}), y_i)$ and $\delta_2 = \argmax_{\delta \in \Delta}\ell(f_{\bm{\theta}}(\mathbf{x_i}), (1 - 2y_i))$ are the solutions to the inner maximization using standard AT (\equref{eq:adversarial_loss1}) and A3T (\equref{eq:our_loss}), respectively. We prove that
\[
\abs{\bm{\theta} (\mathbf{x}_i + \bm{\delta}_{2}) + b} \leq 
\abs{\bm{\theta} (\mathbf{x}_i + \bm{\delta}_{1}) + b}.
\]
\end{theorem*}
\begin{proof}
First note that if $(x_i,y_i)$ is a training example with
$y_i \in \{+1,-1\}$ then the misclassified label is
 $1 - 2y_i$. Also, note that the distance between a data point $\mathbf{x_0}$ and
the decision boundary $\bm{\theta }\mathbf{x} + b$ is given by
$\frac{\abs{\bm{\theta }\mathbf{x_0} + b}}{\|\bm{\theta}\|}$

\noindent
Based on the assumption that $x_i$ is misclassifed then
\[
L(y_i(\bm{\theta }\mathbf{x_i} + b)) \geq 
L((1 - 2y_i)(\bm{\theta }\mathbf{x_i} + b))
\]
Since $L$ is logistic loss it is a monotonically decreasing
function, therefore
\[
(1 - 2y_i)(\bm{\theta }\mathbf{x_i} + b)
\leq 
y_i(\bm{\theta }\mathbf{x_i} + b)
\]
Now, \\
\begin{align*}
& \abs{\bm{\theta} (\mathbf{x}_i + \bm{\delta}_{2}) + b} & \\
= & \abs{\bm{\theta }(x_i - \epsilon(1-2y_i)\sgn(\bm{\theta})) + b} & \\
= & \abs{\bm{\theta }x_i + b - \epsilon(1-2y_i)\|\bm{\theta}\|_1} & \because \bm{\theta} \sgn(\bm{\theta}) = \|\bm{\theta}\|_{1} \\
= & \abs{1-2y_i}\abs{\bm{\theta }x_i + b - \epsilon(1-2y_i)\|\bm{\theta}\|_1} & \because \abs{1-2y_i} = 1\\
= & \abs{(1-2y_i)(\bm{\theta }x_i + b) - \epsilon\|\bm{\theta}\|_1 } & \because \abs{ab} = \abs{a}\abs{b} \wedge (1 - 2y_{i})^2 = 1\\
\leq &
\abs{(y_i(\bm{\theta }x_i + b) - \epsilon y_i^2\|\bm{\theta}\|_1 } & \because \mbox{assumption} \wedge y_{i}^2 = 1  \\
= &
\abs{\bm{\theta }(x_i +  - \epsilon y_i\sgn{(\bm{\theta})) + b }}
& \because  \abs{ab} = \abs{a}\abs{b} \wedge \abs{y_i} = 1 \\
= & \abs{\bm{\theta} (\mathbf{x}_i + \bm{\delta}_{1}) + b} & \\
\end{align*}
 \end{proof}
 
 \begin{algorithm}[!tb]
\caption{Algorithm of A3T$^+$ methods}\label{a1:a3t+}
\begin{algorithmic}[1]
\Require $E$: the number of epochs, $D=\{(x_i, y_i)\}_{i=1}^n$:  the dataset, $f_{\bm{\theta}}(x)$: the machine learning model parametrized by $\bm{\theta}$, $\delta$: the perturbation initialized by $\sigma$ and limited by $\epsilon$, $\tau$: the global learning rate, $\alpha$: the adversarial learning rate, $S$: the number of PGD step, $\Pi$ the projection function, $c$ is the ALS weighting factor, and $\eta$ is the MM step size.
\Ensure Model parameters $\theta$

\Procedure{$A3T^+$}{}
    \State Initial every sample's $\epsilon_i$ with 0
    \For{$e=1,..,E$}
      \For{$(x_i,y_i) \in \mathcal{D}$}
            \State $\hat{y_i} \leftarrow \argmax \left(f(x_i, \theta)\right)$
            \State $\tilde{y_i} \leftarrow (1-c\epsilon_i) \hat{y_i}+(1-c\epsilon_i)\text{Dirichlet}(\beta)$ \Comment{ALS}
            \State $\epsilon_i \leftarrow \epsilon_i + \eta$ \Comment{MM}
            \State $\delta_i \sim \mathcal{N}(0,\sigma^{2}I)$
            \For{$s=1,..,S$}
                \State $\delta_i \leftarrow \Pi_{\epsilon_i}\left(\delta_i + \alpha \cdot \sgn\left(\nabla_{x_i}\ell(f_\theta(x_i+\delta_i),\tilde{y_i})\right)\right)$
            \EndFor
            \If{$\argmax\left(f_\theta(x_i+\delta_i))\right) \neq y_i$} \Comment{MM}
                \State $\epsilon_i \leftarrow \epsilon_i - \eta$
            \EndIf
            \State $\epsilon_i \leftarrow \min(\epsilon_{max},\epsilon_i)$ \Comment{MM}
            \State $\tilde{y_i} \leftarrow (1-c\epsilon_i) y_i+(1-c\epsilon_i)\text{Dirichlet}(\beta)$ \Comment{ALS}
            \State $\theta \leftarrow \theta - \tau \nabla_\theta \ell(f_\theta(x_i+\delta_i),\tilde{y_i})$ 
      \EndFor
    \EndFor
\EndProcedure
\end{algorithmic}
\end{algorithm}

\end{document}